\documentclass[preliminary,creativecommons,sharealike]{eptcs}
\providecommand{\event}{TARK 2023} %

\usepackage{iftex}

\ifpdf
  \usepackage{underscore}          %
  \usepackage[T1]{fontenc}         %
\else
  \usepackage{breakurl}            %
\fi

\usepackage[utf8]{inputenc}
\usepackage{float}
\usepackage{amsfonts}
\usepackage{amsmath}
\usepackage{amsthm}
\usepackage{bbm}
\usepackage{graphicx}
\usepackage{mathtools}

\newcommand{\F}{\mathcal{F}_\varphi}

\newcommand{\card}[1]{\left| #1 \right|}

\renewcommand{\d}[1]{{d\left({#1}\right)}}
\newcommand{\depth}[1]{{d\left({#1}\right)}}
\newcommand{\E}[2]{E_{#1}^{#2}}
\renewcommand{\P}[2]{P_{#1}^{#2}}
\newcommand{\Paphi}{\P{a}{\depth{\varphi}}}

\newcommand{\KI}{\eqref{eq:KP}}
\newcommand{\KIp}{\eqref{eq:KP'}}
\newcommand{\TAz}{\eqref{eq:TA0}}
\newcommand{\TA}{\eqref{eq:TA}}

\newcommand{\N}{\mathbb{N}}
\newcommand{\Z}{\mathbb{Z}}
\renewcommand{\L}{\mathcal{L}}
\newcommand{\agents}{\mathcal{A}}
\newcommand{\atoms}{\mathcal{P}}
\newcommand{\states}{\mathcal{S}}
\renewcommand{\H}{\mathcal{H}}
\newcommand{\Ha}{\mathcal{H}_a}

\newcommand{\DPAL}{\textbf{DPAL}}
\newcommand{\DBEL}{\textbf{DBEL}}
\newcommand{\GL}{\( {GL}_{EF} \)}

\newcommand{\TSo}{\textbf{EDPAL}}
\newcommand{\TSt}{\textbf{ADPAL}}

\newcommand{\Mn}{M_n}
\newcommand{\Mhn}{\hat{M}_n}

\theoremstyle{definition}
\newtheorem{defi}{Definition}
\theoremstyle{plain}
\newtheorem{thm}{Theorem}[section]
\newtheorem{prop}[thm]{Proposition}
\newtheorem{lmm}[thm]{Lemma}
\theoremstyle{remark}

\title{Depth-bounded Epistemic Logic}
\author{Farid Arthaud
\institute{MIT\\ Cambridge, Massachusetts}
\email{farto@csail.mit.edu}
\and
Martin Rinard
\institute{MIT\\ Cambridge, Massachusetts}
\email{rinard@csail.mit.edu}
}

\def\titlerunning{Depth-bounded Epistemic Logic}
\def\authorrunning{F. Arthaud \& M. Rinard}

\begin{document}

\maketitle

\begin{abstract}
	Epistemic logics model how agents reason about their beliefs and the
	beliefs of other agents. Existing logics typically assume the ability of
	agents to reason perfectly about propositions of unbounded modal depth.
	We present~\DBEL{}, an extension of \textbf{S5} that models agents that
	can reason about epistemic formulas only up to a specific modal depth. To
	support explicit reasoning about agent depths, \DBEL{} includes depth
	atoms \( \E{a}{d} \) (agent \( a \) has depth exactly  \( d \)) and \(
	\P{a}{d} \) (agent \( a \) has depth at least \( d \)). We provide a
	sound and complete axiomatization of~\DBEL{}.

	We extend \DBEL{} to support public announcements for bounded depth
	agents and show how the resulting \DPAL{} logic generalizes standard
	axioms from public announcement logic. %
	We present two alternate extensions and identify two undesirable
	properties, amnesia and knowledge leakage, that these extensions have but
	\DPAL{} does not. We provide axiomatizations of these logics as well as
	complexity results for satisfiability and model checking.

	Finally, we use these logics to illustrate how agents with bounded modal
	depth reason in the classical muddy children problem, including upper and
	lower bounds on the depth knowledge necessary for agents to successfully
	solve the problem.
\end{abstract}

\section{Introduction}
Epistemic logics model how agents reason about their beliefs and the beliefs of
other agents. These logics generally assume the ability of agents to perfectly
reason about propositions of unbounded modal depth, which can be seen
as unrealistic in some
contexts~\cite{DBLP:journals/corr/abs-2008-08849,DBLP:journals/jolli/VerbruggeM08}.

To model agents with the ability to reason only to certain preset modal depths,
we extend the syntax of epistemic logic
\textbf{S5}~\cite{DBLP:books/mit/FHMV1995} to depth-bounded epistemic logic
(DBEL).
The \DBEL{} semantics assigns each agent a depth in each state.
For an agent to know a formula \( \psi \) in a given state of a model, the
assigned depth of the agent must be at least the modal depth of \( \psi \), i.e.
\( \d{\psi} \).
To enable agents to reason about their own and other agents' depths, \DBEL{}
includes \textbf{depth atoms} \( \E{a}{d} \) (agent \( a \) has depth exactly  \( d
\)) and \( \P{a}{d} \) (agent \( a \) has depth at least \( d \)).
For example, the formula \( K_a( \P{b}{5} \to K_b p ) \) expresses the fact that,
``agent \( a \) knows that whenever agent \( b \) is depth at least \( 5 \),
agent \( b \) knows the fact \( p \).''
Depth atoms enable agents to reason about agent depths and their consequences
in contexts in which each agent may have complete, partial, or even no
information about agent depths (including its own depth).

We provide a sound and complete axiomatization of \DBEL{}
(Section~\ref{sec:dbel}), requiring a stronger version of the \textsc{Lindenbaum}
lemma which ensures each agent can be assigned a depth (proven in
Appendix~\ref{app:lind}).
Its satisfiability problem for two or more agents is immediately
\textsf{PSPACE}-hard (because \DBEL{} includes \textbf{S5} as a syntactic
fragment).
We provide a depth satisfaction algorithm for \DBEL{} in \textsf{PSPACE}
(Section~\ref{sec:complex}), establishing that the \DBEL{} satisfiability problem
is \textsf{PSPACE}-complete for two or more agents.

Public announcement logic (PAL)~\cite{DBLP:journals/jolli/GerbrandyG97} extends
epistemic logic with public announcements.
PAL includes the following public announcement and knowledge
axiom~\eqref{eq:intro}, which characterizes agents' knowledge after public
announcements,
\begin{equation}\label{eq:intro}
	[\varphi] K_a \psi \leftrightarrow (\varphi \to K_a [\varphi] \psi).
	\tag{PAK}
\end{equation}
We extend \DBEL{} to include public announcements (Section~\ref{sec:dpal}).
The resulting depth-bounded public announcement logic (DPAL) provides a semantics
for public announcements in depth-bounded epistemic logic, including a
characterization of how agents reason when public announcements exceed their
epistemic depth.
We prove the soundness of several axioms that generalize~\eqref{eq:intro} to
\DPAL{}, first in a setting where each agent has exact knowledge of its own
depth, then in the general setting where each agent may have partial or even no
knowledge of its own depth.
We provide a sound axiom set for \DPAL{} as well as an upper bound on the
complexity of its model checking problem~\footnote{Arthaud and
Rinard~\cite{DBLP:journals/corr/abs-2305-08607} present a lower bound for this
problem, as well as additional results, proofs and content.}

We also present two alternate semantics that extend \DBEL{} with public
announcements (Section~\ref{subsec:s12}).
The resulting logics verify simpler generalizations of~\eqref{eq:intro} in the
context of depth-bounded agents, but each has one of two undesirable properties
that we call \textit{amnesia} and \textit{knowledge leakage}.
Amnesia causes agents to completely forget about all facts they knew after
announcements, whereas knowledge leakage means shallow agents can infer
information from what deeper agents have learned from a public announcement.
\DPAL{} suffers from neither of these two undesirable properties.
We provide a sound and complete axiomatization of the first of the two alternate
semantics (Section~\ref{sec:ax}).
We also prove the \textsf{PSPACE}-completeness of its satisfiability problem and
show that its model checking problem remains \textsf{P}-complete
(Section~\ref{sec:complex}).

Finally, we use these logics to illustrate how agents with bounded depths reason
in the muddy children reasoning problem~\cite{DBLP:books/mit/FHMV1995}.
We prove a lower bound and an upper bound on the structure of knowledge of depths required for agents to solve this
problem (Section~\ref{sec:muddy}).

\paragraph{Related work}
Logical omniscience, wherein agents are capable of deducing any fact deducible
from their knowledge, is a well-known property of most epistemic logics. The
ability of agents to reason about facts to unbounded modal depth is a
manifestation of logical omniscience. Logical omniscience has been viewed as
undesirable or unrealistic in many contexts~\cite{DBLP:books/mit/FHMV1995} and
many attempts have been made to mitigate or eliminate
it~\cite{DBLP:books/mit/FHMV1995,meyer2003modal,DBLP:journals/ijis/Sim97}.
To the best of our knowledge, only
Kaneko and Suzuki~\cite{DBLP:conf/aiml/KanekoS00} below have involved modal depth
in the treatment of logical omniscience in epistemic logic.

Kaneko and Suzuki~\cite{DBLP:conf/aiml/KanekoS00} define the logic of shallow
depths~\GL{}, which relies on a set \( E \) of chains of agents \( (i_1, \ldots,
i_k) \) for which chains of modal operators \( K_{i_1} \cdots K_{i_m} \) can
appear.
A subset \( F \subseteq E \) restricts chains of modal operators along
which agents can perform deductions about other agents' knowledge.
An effect of bounding agents' depths in~\DPAL{} is creating a set of allowable
chains of modal operators \( \cup_a \{ (a, i_1, \ldots, i_{d_a}), \; (i_1,
\ldots, i_{d_a}) \in \agents^{d_a} \} \).
Unlike~\GL{}, the bound on an agent's depth is not global in~\DPAL{}, it can
also be a function of the worlds in the Kripke possible-worlds semantics~\cite{DBLP:books/mit/FHMV1995}.
In particular, \DPAL{}, unlike~\GL{}, enables agents to reason about their own
depth, the depth of other agents, and (recursively) how other agents reason about
agent depths.
\DPAL{} also includes public announcements, which to the best of our knowledge
has not been implemented in~\GL{}.

Kline~\cite{kline2013evaluations} uses~\GL{} to investigate the \( 3 \)-agent
muddy children problem, specifically by deriving minimal epistemic structures \(
F \) that solve the problem.
The proof relies on a series of
belief sets with atomic updates called ``resolutions,'' with the nested length of
the chains in \( F \) providing epistemic bounds on the required reasoning.
\DPAL{}, in contrast, includes depth atoms and public announcements as
first-class features.
We leverage these features to directly prove theorems expressing that for \( k \)
muddy children,
\textit{(i)} (Theorem~\ref{thm:lowbound}) if the problem is solvable by an agent,
that agent must have depth at least \( k-1 \) and know that it has depth at least
\( k-1 \) (this theorem provides a lower bound on the agent depths required to
solve the problem)
and \textit{(ii)} (Theorem~\ref{thm:uppbound}) if an agent has depth at least \(
k-1 \), knows it, knows another agent is depth at least \( k-2 \), knows that the
other agent knows of another agent of depth \( k-2 \), \textit{etc}., then it can
solve the problem (this theorem provides an upper bound on the agent depths
necessary to solve the problem).
Our depth bounds match the depth bounds of Kline~\cite{kline2013evaluations} for
\( 3 \) agents (Theorems~3.1 and~3.3 in~\cite{kline2013evaluations}), though our
bounds also provide conditions on recursive knowledge of depths for the agents as
described above.

Dynamic epistemic logic
(DEL)~\cite{DBLP:journals/synthese/BaltagM04,van2007dynamic} introduces more
general announcements. Private announcements are conceptually similar to public
announcements in \DPAL{} in that they may be perceived by only some of the
agents. In DEL, model updates depend only on the relation between states in the
initial model and the relations in the action model. But in \DPAL{}, model
updates must also take into account the agent depths in the entire connected components of
each state (see Definition~\ref{defi:dpal}).

Resource-bounded agents in epistemic logics have been explored by Balbiani et.\
al~\cite{DBLP:conf/atal/BalbianiDL16} (limiting perceptive and inferential
steps), Artemov and Kuznets~\cite{DBLP:conf/tark/ArtemovK09} (limiting the
computational complexity of inferences), and Alechina et.\
al~\cite{DBLP:conf/kramas/AlechinaLNR08} (bounding the size of the set of
formulas an agent may believe at the same time and introducing communication
bounds).
Alechina et.\ al~\cite{DBLP:conf/kramas/AlechinaLNR08} also bound the modal depth
of formulas agents may believe, but all agents share the same depth bound and
they leave open the question of whether inferences about agent depth or memory
size could be implemented, which~\DPAL{} does.

\section{Depth-bounded epistemic logic}\label{sec:dbel}
The modal depth \( \d{\varphi} \) of a formula \( \varphi \), defined as the
largest number of modal operators on a branch of its syntactic tree, is the
determining factor of the complexity of a formula in depth-bounded epistemic
logic (DBEL).
Modal operators are the main contributing factor to the complexity of model
checking a formula; the recursion depth when checking satisfiability of a formula
is equivalent to its modal depth~\cite{DBLP:conf/atal/Lutz06}; and bounding modal
depth often greatly simplifies the complexity of the satisfiability problem in
epistemic logics~\cite{DBLP:conf/aiml/Nguyen04}.
Humans are believed to reason within limited modal depth~\cite{DBLP:journals/corr/abs-2008-08849,DBLP:journals/jolli/VerbruggeM08}.

We extend the syntax of classical epistemic logic by assigning to each agent \( a
\) in a set of agents \( \agents \) a depth \( d(a, s) \) in each possible world
\( s \). The language also includes \textbf{depth atoms} \( \E{a}{d} \) and \(
\P{a}{d} \) to respectively express that agent \( a \) has depth exactly \( d \)
and agent \( a \) has depth at least \( d \).

To know a formula \( \varphi \), agents are required to be at least as deep as \(
\d{\varphi} \) and also know that the formula \( \varphi \) is true in the usual
possible-worlds semantics sense~\cite{DBLP:books/mit/FHMV1995}.
We translate the classical modal operator \( K_a \) from multi-agent epistemic
logic into the operator \( K^\infty_a \) with the same properties, therefore \(
K^\infty_a \varphi \) can be interpreted as ``agent \( a \) would know \( \varphi
\) if \( a \) were of infinite depth''.
The operator \( K_a \varphi \) will now take the meaning described above, i.e. \(
\Paphi \wedge K^\infty_a \varphi \).
\begin{defi}\label{defi:l}
	The language of \DPAL{} is inductively defined as, for all agents \( a
	\in \agents \) and depths \( d \in \N \),
	\[
	\L^\infty := \varphi = p \; | \; \E{a}{d} \; | \; \P{a}{d} \; | \; \neg
	\varphi \; | \; \varphi \wedge \varphi \; | \; K_a \varphi \; | \;
	K^\infty_a \varphi \; | \; [\varphi] \varphi.
	\]
	The \( K^\infty_a \) operator is used mainly as a tool in axiomatization
	proofs, we call \( \L \) the fragment of our logic formulas without any \( K^\infty_a \) operators, which will be used in most of our theorems.
	We further define \( \H^\infty \) and \( \H \) to respectively be the
	syntactic fragments of \( \L^\infty \) and \( \L \) without public
	announcements \( [\varphi] \psi \).

	The modal depth \( d \) of a formula in \( \L^\infty \) is inductively
	defined as,
	\begin{gather*}
		\depth{p} = \depth{\E{a}{d}} = \depth{\P{a}{d}}  = 0 \qquad \quad
		\depth{\neg \varphi} = \depth{\varphi} \qquad \quad
		\depth{[\varphi] \psi} = \depth{\varphi} + \depth{\psi} \\
		\depth{\varphi \wedge \psi} = \max \left( \depth{\varphi}, \depth{\psi} \right) \qquad \quad
		\depth{K_a \varphi} = 1 + \depth{\varphi} \qquad \quad
		\depth{K^\infty_a \varphi} = 1 + \depth{\varphi}.
	\end{gather*}
\end{defi}

We defer treatment of public announcements \( [\varphi] \psi \) to
Section~\ref{sec:dpal}.
We work in the framework of \textbf{S5}~\cite{DBLP:books/mit/FHMV1995}, assuming
each agent's knowledge relation to be an equivalence relation, unless otherwise
specified---however, our work could be adapted to weaker epistemic
logics~\cite{DBLP:books/mit/FHMV1995} by removing the appropriate axioms.
\begin{defi}\label{defi:dbel}
	A model in \DBEL{} is defined as a tuple \( M = (\states, \sim, V, d) \)
	where \( \states \) is a set of states, \( V: \states \to 2^{\atoms} \)
	is the valuation function for atoms and \( d: \agents \times \states \to
	\N \) is a depth assignment function.
	For each agent \( a \), \( \sim_a \) is an equivalence relation on \(
	\states \) modeling which states are seen as equivalent in the eyes of \(
	a \).
	The semantics are inductively defined over \( \H^\infty \) by,
\begin{gather*}
	(M, s) \models p \iff p \in V(s)
	\qquad \quad
	(M, s) \models E_a^d \iff d(a, s) = d
	\qquad \quad
      	(M, s) \models \P{a}{d} \iff d(a, s) \geq d \\
	(M, s) \models \neg \varphi \iff (M, s) \not\models \varphi
	\qquad \quad
	(M, s) \models \varphi \wedge \psi \iff (M, s) \models \varphi \text{ and
	} (M, s) \models \psi \\
	(M, s) \models K^\infty_a \varphi \iff (\forall s', \; s \sim_a s'
	\implies (M, s') \models \varphi)
	\qquad
	(M, s) \models K_a \varphi \iff (M, s) \models \Paphi \wedge K^\infty_a
	\varphi.
\end{gather*}
\end{defi}
Note that this definition does not require agents to have any (exact or
approximate) knowledge of their own depth.
On the other hand, it does not prohibit agents agents from having exact knowledge
of their own depths, for instance we could model each agent carrying out some
`meta-reasoning' about its own depth~\footnote{For instance deducing \( \Paphi \)
from the fact that it knows \( \varphi \), or deducing \( \neg \P{a}{n} \) from
the fact that it does not know \( K^{n}_a \top \).} leading each agent to know
its own depth exactly.
These models are a subset of the class of the models we consider, which we study
in more detail in Section~\ref{sub:unam}.

\begin{table}
\centering
 \begin{tabular}{|| r | l ||}
 \hline
 All propositional tautologies & \( p \to p \), \textit{etc}. \\
 Deduction & \( (K_a \varphi \wedge K_a (\varphi \to \psi)) \to K_a \psi \) \\
 \hline
 Truth & \( K_a \varphi \to \varphi \) \\
 Positive introspection & \( ( K_a \varphi \wedge \P{a}{\depth{\varphi} + 1} ) \to K_a (\P{a}{\depth{\varphi}} \to K_a \varphi) \) \\
 Negative introspection & \( ( \neg K_a \varphi \wedge \P{a}{\depth{\varphi} + 1} ) \to K_a \neg K_a \varphi \) \\
 \hline
 Depth monotonicity & \( \P{a}{d} \to \P{a}{d-1} \) \\
 Exact depths & \( \P{a}{d} \leftrightarrow \neg (E_a^0 \vee \cdots \vee E_a^{d-1})  \) \\
 Unique depth & \( \neg(E_a^{d_1} \wedge E_a^{d_2}) \) for \( d_1 \neq d_2 \) \\
 \hline
 Depth deduction & \( K_a \varphi \to \P{a}{\depth{\varphi}} \) \\
 \hline
 \hline
 \textit{Modus ponens} & From \( \varphi \) and \( \varphi \to \psi \), deduce \(
 \psi \) \\
 Necessitation & From \( \varphi \) deduce \( \P{a}{\depth{\varphi}} \to
 K_a \varphi \) \\
 \hline
 \end{tabular}
 \caption{Sound and complete axiomatization for \DBEL{} over \( \H \).}\label{tab:axb}
\end{table}

As \DBEL{} is an extension of \textbf{S5} up to renaming of the modal operators,
one can expect for it to have a similar axiomatization: one new axiom is needed
to axiomatize \( K_a \) and three others for depth atoms.
\begin{thm}\label{thm:axb}
	Axiomatization from Table~\ref{tab:axb} is sound and complete with
	respect to \DBEL{} over \( \H \).
\end{thm}
\begin{proof}
	Rather than directly showing soundness and completeness, we show it is
	equivalent to the axiomatization of Table~\ref{tab:axb2} in
	Appendix~\ref{app:ax} on the
	fragment \( \H \), which is shown to be sound and complete over \(
	\H^\infty \) in Theorem~\ref{thm:altaxb}.
	We begin by proving any proposition in \( \H \) that can be shown using
	Table~\ref{tab:axb} can be shown using Table~\ref{tab:axb2} and
	then that any proof of a formula in \( \H \) using the axioms in
	Table~\ref{tab:axb2} can be shown using those in Table~\ref{tab:axb}.

	For the first direction, we prove that the axioms in Table~\ref{tab:axb}
	can be proven using those from Table~\ref{tab:axb2}.
Most of them are immediate applications of bounded knowledge within the axioms of
Table~\ref{tab:axb2}, along with tautologies when necessary.
For positive and negative introspection, see equation~\eqref{eq:posinscompl}
below in the proof of the opposite direction of the equivalence.
We prove the least evident axiom, the deduction axiom, here as an example:
\begingroup
\allowdisplaybreaks
\begin{align}
 \text{Deduction} \quad
        & (K^\infty_a \varphi \wedge K^\infty_a (\varphi \to \psi)) \to
        K^\infty_a \psi \label{eq:st1} \\
 \text{Bounded knowledge in~\eqref{eq:st1}} \quad
        & (K^\infty_a \varphi \wedge K^\infty_a (\varphi \to \psi)) \to
        \P{a}{\depth{\psi}} \to K_a \psi \label{eq:st2} \\
 \text{Tautology in~\eqref{eq:st2}} \quad
        & \P{a}{\max(\depth{\varphi}, \depth{\psi})} \to K^\infty_a \varphi \to
        K^\infty_a (\varphi \to \psi) \to \P{a}{\depth{\psi}} \to K_a \psi
        \label{eq:st3} \\
 \text{Repeated depth consistency} \quad
        & \P{a}{\max(\depth{\varphi}, \depth{\psi})} \to (\Paphi \wedge
        \P{a}{\depth{\psi}}) \label{eq:st4} \\
 \text{Bounded knowledge and~\eqref{eq:st3} and~\eqref{eq:st4}} \quad
        & \P{a}{\max(\depth{\varphi}, \depth{\psi})} \to K_a \varphi \to
        K_a (\varphi \to \psi) \to K_a \psi \label{eq:st5} \\
 \text{Bounded knowledge in~\eqref{eq:st5}} \quad
        & K_a \varphi \to K_a (\varphi \to \psi) \to K_a \psi. \notag
\end{align}
\endgroup

	In the other direction, we will show by induction over a proof of a valid
	formula in \( \H \) using Table~\ref{tab:axb2} that it can be transformed
	into a proof with the same conclusion, using only axioms from
	Table~\ref{tab:axb}.
The transformation of a proof in the first axiomatization is as follows,
\begin{itemize}
\item If an item of the proof is a propositional tautology, replace all \(
K^\infty_a \varphi \) subformulas by \( \Paphi \to K_a \varphi \), clearly the
tautology still holds and it is in Table~\ref{tab:axb}.
\item If an item is an instance of the bounded knowledge axiom, replace it with
the formula \\
\(
K_a \varphi \leftrightarrow (\Paphi \wedge \Paphi \to K_a \varphi)
\)
which is a consequence of depth deduction and a tautology
(and therefore can be added to the proof with two extra steps).
\item If it uses any of the other axioms, replace it with the corresponding axiom
(with the same name) from Table~\ref{tab:axb}.
\end{itemize}

We now have a sequence that has the same conclusion (since the conclusion is in
\( \H \)) and only uses axioms from Table~\ref{tab:axb}.
The last thing to show for this to be a proof in this axiomatization is that all
applications of \textit{modus ponens} and necessitation are still correct within
this sequence.
To this end, we show by induction that each step of the sequence is the same as
the original proof where every \( K^\infty_a \varphi \) subformula in each step
has been replaced by \( \Paphi \to K_a \varphi \).

First, note that this is the case for the two first bullet points of our
transformation rules above.
This is also true of each axiom in the table after our transformation: a proof
similar to the one in equation~\eqref{eq:st1} will yield the equivalence for
deduction, the only remaining non-trivial cases are positive and negative
introspection.
For positive introspection, performing the substitution yields,
\begin{equation}\label{eq:posinscompl}
 (\Paphi \to K_a \varphi) \to \P{a}{\depth{\varphi} + 1} \to K_a(\Paphi \to K_a
 \varphi).
\end{equation}
Through application of a tautology and the depth monotonicity axiom we find it to
be equivalent to,
\(
 \P{a}{\depth{\varphi} + 1} \to K_a \varphi \to K_a(\Paphi \to K_a
 \varphi).
\)
Therefore, up to adding steps to the proof and using tautologies, we can prove
the axiom from Table~\ref{tab:axb} from the axiom in Table~\ref{tab:axb2} after
the substitution.
The same can be said of negative introspection through a similar transformation.

Finally, since \textit{modus ponens} and necessitation also maintain the property
of replacing \( K^\infty_a \varphi \) subformulas in each step by \( \Paphi \to
K_a \varphi \), it is true that the transformed proof is indeed a proof of the
same conclusion in Table~\ref{tab:axb}'s axiomatization.
\end{proof}

\section{Depth-bounded public announcement logic}\label{sec:dpal}
We next present how to incorporate depth announcements in \DBEL{}, which are a
key challenge in defining depth-bounded public announcement logic (DPAL).
Recall the axiom~\eqref{eq:intro} of public announcement logic, \( [\varphi] K_a
\psi \leftrightarrow (\varphi \to K_a [\varphi] \psi) \).
For the right-hand side to be true, agent \( a \) must be of depth \( d([\varphi]
\psi) = d(\varphi) + d(\psi) \) according to~\DBEL{}.
This suggests that an agent must ``consume'' \( \d{\varphi} \) of its depth every
time an announcement \( \varphi \) is made, meaning that an agent's depth behaves
like a depth budget with respect to public announcements.

Moreover, to model that some agents might be too shallow for the announcement
\( \varphi \), each possible world is duplicated in a \textit{negative} version
where the announcement has not taken place and a \textit{positive} version where
the announcement takes place in the same way as in PAL\@.
Agents who are not deep enough to perceive the announcement see the negative and
positive version of the world as equivalent.
\begin{defi}\label{defi:dpal}
	Models in \textbf{depth-bounded public announcement logic} (DPAL) are
	defined the same way as in \DBEL{} and the semantics is extended to \(
	\L^\infty \) by
	\( (M, s) \models [\varphi] \psi \iff ((M, s) \models \varphi \implies \\
	(M \mid \varphi, (1, s)) \models \psi) \),
	where we define \( M \mid \varphi \) to be the model \( (\states',
	\sim', V', d') \), where,
\begingroup
\allowdisplaybreaks
\begin{align}
\states' & = (\{ 0 \} \times \states) \cup \{ (1, s), \; s \in \states, \; (M, s)
	\models \varphi \} \notag \\
	\sim'_a & \text{ is the transitive symmetric closure of } R_a \text{ such that},
	\notag \\
	(i, s) \, R_a \, (i, s') & \iff s \sim_a s' \qquad \text{ for } i=0,1 \notag \\
(1, s) \, R_a \, (0, s) & \iff (M, s) \not\models \P{a}{\depth{\varphi}} \notag \\
V'((i, s)) & = V(s) \qquad \qquad \; \, \text{ for } i=0,1 \notag \\
d'(a, (0, s)) & = d(a, s) \notag \\
d'(a, (1, s)) & = \begin{cases}
d(a, s) \quad \text{if } d(a, s) < \depth{\varphi} \\
d(a, s) - \depth{\varphi} \quad \text{otherwise.}
\end{cases}\label{eq:s3depth}
\end{align}
\endgroup
\end{defi}

Since public announcements are no longer unconditionally and universally heard by
all agents, we revisit the axiom~\eqref{eq:intro} in~\DPAL{}.
The determining factor is \textbf{depth ambiguity}: agents that are unsure about
their own depth introduce uncertainty about which agents have perceived the
announcement.

\subsection{Unambiguous depths setting}\label{sub:unam}
A model verifies the \textbf{unambiguous depths} setting whenever each agent
knows its own depth exactly:
\begin{equation}\label{eq:unam}
	\forall a, s, s', \quad s \sim_a s' \implies d(a, s) = d(a, s').
\end{equation}

The proof of the following theorem is given as Proposition~\ref{prop:col3app}
in Appendix~\ref{app:props}.
\begin{thm}\label{thm:col3}
	For all \( \varphi \in \L^\infty \), the following two properties,
	respectively called \textbf{knowledge preservation} and
	\textbf{traditional announcements}, are valid in~\DPAL{} in the
	unambiguous depths setting,
\begin{alignat}{3}
& \forall \psi \in \L_a^\infty, \; \; && \neg \P{a}{\depth{\varphi}} && \to
	\left( [\varphi] K_a \psi \leftrightarrow (\varphi \to K_a \psi) \right)
	\tag{KP} \label{eq:KP} \\
& \forall \psi \in \L^\infty, \; \; && \P{a}{\depth{\varphi}} && \to
	\left( [\varphi] K_a \psi \leftrightarrow \left( \varphi \to K_a
	[\varphi] \psi \right) \right), \tag{TA} \label{eq:TA0}
\end{alignat}
	where \( \L_a^\infty \) is the fragment of \( \L^\infty \) without depth atoms or modal operators for agents other than \( a \).
\end{thm}

\paragraph{Discussion}
Knowledge preservation~\KI{} means that an agent who is not deep enough to
perceive an announcement \( \varphi \) must not change its knowledge of a formula
\( \psi \).
However, such a property could not be true of all formulas \( \psi \), for
instance if \( \psi = K_a K_b p \) but \( b \) is deep enough to perceive \(
\varphi \), then the depth adjustment formula~\eqref{eq:s3depth} could mean that
\( b \)'s depth is now \( 0 \), making \( \psi \) no longer hold.
Even when \( a \) is certain about \( b \)'s depth, its uncertainty about what
the announcement entails could also mean that formulas such as \( \neg K_b p \)
could no longer be true if \( \P{b}{\d{\varphi}} \) and \( \varphi \to p \) in
the model.
This demonstrates that in depth-bounded logics public
announcements must introduce uncertainty: if \( a \) is unsure what \( b \) has
perceived, it can no longer hold any certainties about what \( b \) does not
know.
This is not the case in PAL since all agents perceive all announcements.
Our treatment of the depth-ambiguous case in Section~\ref{sec:amb} generalizes~\KI{} to obtain a property~\KIp{} that holds on all formulas in \( \L^\infty \).

Traditional announcements~\TAz{} ensures that announcements behave the same as in
PAL when the agent is deep enough for the announcement.
The caveats from the discussion of~\KI{} no longer apply here, as any \( K_b \)
operator that appears in \( \psi \) will still appear after the same public
announcement operator, meaning that depth variations or knowledge variations are
accounted for.

\subsection{Ambiguous depths setting}\label{sec:amb}
We now abandon the depth unambiguity assumption from equation~\eqref{eq:unam},
and explore how properties~\KI{} and~\TAz{} generalize to settings without depth
unambiguity.
We find a condition that ensures that sufficient knowledge about
other agents' depths is given to \( a \) in order to maintain its recursive
knowledge about other agents.
The proof to the following theorem is given as Proposition~\ref{prop:propsstrapp}
in Appendix~\ref{app:props}.
\begin{thm}\label{thm:col3str}
For any \( \varphi \in \L^\infty \), let \( \F: \L^\infty \to \L^\infty \) be
inductively defined as,
\begin{gather*}
	\F(p) = \F(\E{a}{d}) = \F(\P{a}{d}) = \top \qquad \quad \F(\neg \psi) = \F(\psi) \qquad \quad
	\F(\psi \wedge \chi) = \F(\psi) \wedge \F(\chi) \\
	\F(K_a \psi) = \neg K^\infty_a(\varphi \to \Paphi) \wedge
	K^\infty_a (\varphi \to \neg \Paphi \vee \P{a}{\d{\varphi} + \d{\psi}})
	\wedge K^\infty_a \F(\psi) \\
        \F(K^\infty_a \psi) = \neg K^\infty_a(\varphi \to \Paphi) \wedge
        K^\infty_a \F(\psi) \qquad \qquad \qquad \qquad
        \F([\psi_1] \psi_2) = \F(\psi_1) \wedge \F(\psi_2).
\end{gather*}

	For all \( \varphi \in \L^\infty \), the following two properties are
	valid in~\DPAL{},
\begin{alignat}{3}
	& \forall \psi \in \L^\infty, \; \; && \F(K_a \psi) && \to
	\left( [\varphi] K_a \psi \leftrightarrow (\varphi \to
	K_a \psi) \right) \tag{KP'} \label{eq:KP'} \\
	& \forall \psi \in \L^\infty, \; \; && K^\infty_a(\varphi \to \P{a}{\depth{\varphi}}) &&
	\to \left( [\varphi] K_a \psi \leftrightarrow \left(\varphi \to
	K_a [\varphi] \psi \right) \right).
	\tag{TA'} \label{eq:TA}
\end{alignat}
\end{thm}

\subsection{Alternate treatments of model updates for public announcements}\label{subsec:s12}
One question is whether using a definition of public announcements closer to~PAL
would produce a version of the above axioms closer to~\eqref{eq:intro}.
Eager depth-bounded public announcement logic (EDPAL) below unconditionally
decrements the depth value of all agents after public announcements.
\begin{defi}[\TSo{}]\label{defi:sem1}
	\TSo{} extends the \DBEL{} semantics to include public announcements by
	defining
	\( (M, s) \models [\varphi] \psi \iff ((M, s) \models \varphi
	\implies (M \mid \varphi, s) \models \psi) \),
	where \( M \mid \varphi \) is the model \( (\states', \sim', V, d') \) in
	which \( \states' = \{ s \in \states, \; (M, s) \models \varphi \} \), \(
	\sim'_a \) is the restriction of \( \sim_a \) to \( \states' \), \( d'(a,
	s) = d(a, s) - \depth{\varphi} \), and $d$ may take values in \( \Z \).
\end{defi}
\TSo{} has a sound and complete
axiomatization based on the axiomatization of~\DBEL{} (Theorem~\ref{thm:ax1}),
which also allows us to prove the complexity result of Theorem~\ref{thm:sat}.

However, another consequence of its definition is that excessive public
announcements in~\TSo{} can lead an agent to a state in which it cannot reason anymore, as
it has consumed its entire depth budget.
\begin{prop}[Amnesia]\label{prop:amnesia}
In~\TSo{}, the formula \( \neg \P{a}{\depth{\varphi}} \to [\varphi] \neg
K_a \psi \) is valid for all \( \varphi \) and \( \psi \).
\end{prop}
\begin{proof}
	If \( (M,s) \not\models \varphi \) then the implicand is true.
	If \( (M,s) \models \varphi \wedge \neg \Paphi \) then the depth of \( a
	\) in \( (M \mid \varphi, s) \) will be at most \( -1 \), meaning that \(
	(M \mid \varphi, s) \not\models K_a \psi \) for all \( \psi \).
\end{proof}

In particular, for \( \psi = \top \) one notices that standard intuitions about
knowledge fail in~\TSo{}.
This property is undesirable: \textit{(i)} one may expect agents to maintain some
knowledge even after public announcements that they are not deep enough to
understand and \textit{(ii)} deeper agents should be able to continue to benefit
from the state of knowledge of shallower agents even after the shallower agents
have exceeded their depth.

One way to try to remedy this property is to change model updates in~\TSo{} to make agents perceive announcements only when they are deep enough to understand them.
The resulting asymmetric depth-bounded public announcement logic (ADPAL) removes
depth from an agent's budget only when it is deep enough for an announcement,
and only updates its equivalence relation in states where it is deep enough for
the announcement.
\begin{defi}[\TSt{}]\label{defi:sem2}
    \TSt{} extends the \DBEL{} semantics to include public announcements by defining
	\( (M, s) \models [\varphi] \psi \iff ((M, s) \models \varphi \implies (M
	\mid \varphi, s) \models \psi) \),
    where \( M \mid \varphi \) is the model \( (\states, \sim', V, d') \),
\begin{align*}
	s \not\sim'_a s' \iff s \not\sim_a s' \text{ or } & \begin{cases}
(M, s) \models \P{a}{\depth{\varphi}} \\
(M, s) \models \varphi \iff (M, s') \not\models \varphi,
\end{cases} \\
	d'(a, s) = & \begin{cases}
d(a, s) \quad \text{if } d(a, s) < \depth{\varphi} \\
d(a, s) - \depth{\varphi} \quad \text{otherwise.}
\end{cases}
\end{align*}
	The relations \( \sim_a \) are only assumed to be reflexive (as opposed
	to equivalence relations earlier).
\end{defi}

Unfortunately, in~\TSt{} an agent that is too shallow for an announcement could
still learn positive information that was learned by another agent who is deep
enough to perceive the announcement.
We call this property {\em knowledge leakage} as reflected in the following
proposition.
\begin{prop}[Knowledge leakage]\label{prop:leakage}
	\TSt{} does not verify the \( \to \) direction of~\KIp{}.
\end{prop}
\begin{proof}
	Consider three worlds, \( \{0, 1, 2\} \) and three agents \( a, b, c \).
	The relations for \( a \) and \( c \) are identity, the relation for \( b
	\) is the symmetric reflexive closure of, \( 0 \sim_b 1 \sim_b 2 \).
	The depth of \( a \) is \( 1 \) everywhere, \( b \)'s depth is \( 0, 2,
	0 \) in each respective state and the depth of \( c \) is \( 2 \)
	everywhere.
	The atom \( p_0 \) is true only in \( 0 \) and \( 1 \).
	Consider \( \varphi = K_c K_c p_0 \), it is true in \( 0 \) and \( 1 \)
	only, and consider \( \psi = K_b p_0 \).
	\( K_a \psi \) is not true in state \( 1 \), however \( [\varphi] K_a
	\psi \) is.
	Moreover, one can easily check that \( \F(K_a \psi) \) is true in that
	state.
\end{proof}
The proof provides a practical example of such leakage in~\TSt{} and we further demonstrate knowledge leakage in Proposition~\ref{prop:badmuddy} in the muddy children reasoning problem (see Section~\ref{sec:muddy}).

Note how each direction of the equivalence in~\KIp{} expresses (\( \to \)) that no knowledge leakage occurs and (\( \leftarrow \)) no amnesia occurs.
As shown in Theorem~\ref{thm:col3str}, \DPAL{} verifies both directions and thus has neither amnesia nor knowledge leakage.
As reflected in the following proposition, although EDPAL has amnesia, it doesn't
have knowledge leakage and
verifies~\TAz{}.
\begin{prop}\label{prop:col1}\cite{DBLP:journals/corr/abs-2305-08607}
	\TSo{} verifies~\TAz{} and the \( \to \) direction in~\KI{} over \( \psi
	\in \L^\infty \), but not the converse.
\end{prop}

\section{Axiomatizations}\label{sec:ax}
\begin{table}
\centering
 \begin{tabular}{|| r | l ||}
 \hline
 All axioms from Table~\ref{tab:axb} & \\
 \hline
 Atomic permanence & \( [\varphi] p \leftrightarrow (\varphi \to p) \) \\
 Depth adjustment & \( \forall d \in \Z, \; [\varphi] \E{a}{d} \leftrightarrow \left(
	 \varphi \to \E{a}{\depth{\varphi} + d} \right) \) \\
 Negation announcement & \( [\varphi] \neg \psi \leftrightarrow (\varphi \to \neg
	 [\varphi] \psi) \) \\
 Conjunction announcement & \( [\varphi] (\psi \wedge \chi) \leftrightarrow
	 ([\varphi] \psi \wedge [\varphi] \chi) \) \\
 Knowledge announcement & \( [\varphi](\P{a}{\depth{\psi}} \to K_a \psi) \leftrightarrow
	 ( \varphi \to \P{a}{\depth{\varphi} + \depth{\psi}} \to K_a [\varphi] \psi ) \) \\
 Announcement composition & \( [\varphi] [\psi] \chi \leftrightarrow
	 ([\varphi \wedge [\varphi] \psi] \chi) \) \\
 \hline
 \hline
 \textit{Modus ponens} & From \( \varphi \) and \( \varphi \to \psi \), deduce \(
 \psi \) \\
 Necessitation & From \( \varphi \) deduce \( \P{a}{\depth{\varphi}} \to
 K_a \varphi \) \\
 \hline
 \end{tabular}
 \caption{Sound and complete axiomatization of~\TSo{} over \( \L \).}\label{tab:ax1}
\end{table}

\begin{thm}\label{thm:ax1}
	The axiomatization in Table~\ref{tab:ax1} is sound and complete with
	respect to~\TSo{} (Definition~\ref{defi:sem1}) over the fragment \( \L
	\).
\end{thm}
\begin{proof}
	Similarly to the proof of Proposition~\ref{thm:axb}, rather than
	directly showing soundness and completeness we show it is equivalent to
	the axiomatization of Table~\ref{tab:ax4}, which is shown to be sound and
	complete for~\TSo{} in Theorem~\ref{thm:ax4} in Appendix~\ref{app:ax}.

	In the first direction, all axioms in Table~\ref{tab:ax1} can be shown
	using those in Table~\ref{tab:ax4} immediately, either from the proof of
	Proposition~\ref{thm:axb} or because they are the same.
	The only difficulty lies in knowledge announcement, but a proof similar
	to equation~\eqref{eq:st1} shows it is sound.

	The other direction also follows the exact same proof as in
	Proposition~\ref{thm:axb}: the public announcement axioms are direct
	translations of the same axioms in Table~\ref{tab:ax4} by replacing the
	\( K^\infty_a \varphi \) subformulas with \( \Paphi \to K_a \varphi \).
	The proof transformation from Proposition~\ref{thm:axb} therefore
	still yields a proof of the same formula in this axiomatization, which
	proves completeness.
\end{proof}

We now present a sound set of axioms for~\DPAL{}.
The main missing axioms for a sound and complete axiomatization are knowledge and
public announcements, which we explored in the previous section, and announcement
composition.
In fact, announcement composition cannot exist in~\DPAL{}, since making a single
announcement of depth \( d_1 + d_2 \) can behave very differently from making an
announcement of depth \( d_1 \) followed by another of depth \( d_2 \),
for instance when an agent's depth is between \( d_1 \) and \( d_1 + d_2 \).

\begin{thm}\label{thm:ax1sound}
	Replacing knowledge announcement by~\KIp{} and~\TA{}
	and depth adjustment by,
	\[ \forall d \in \N, \quad [\varphi] \E{a}{d} \leftrightarrow \left(
	\varphi \to \left( (\P{a}{\depth{\varphi}} \wedge \E{a}{d +
	\depth{\varphi}}) \vee (\neg \P{a}{\depth{\varphi}} \wedge \E{a}{d})
	\right) \right) \]
	in Table~\ref{tab:ax1} produces a set of sound axioms with respect
	to~\DPAL{}~\footnote{One could also easily add axioms for \( K^\infty_a \) modal operators, for instance using those from Table~\ref{tab:ax4} in Appendix~\ref{app:ax}.}.
\end{thm}
\begin{proof}
Theorem~\ref{thm:col3str} verifies the two axioms~\KIp{} and~\TA{}.
The proofs for most axioms follows from Theorem~\ref{thm:ax1} and that knowledge
is defined the same way in both semantics.
In particular, atomic permanence and conjunction announcement axioms are proven
in Theorem~\ref{thm:col3}'s induction for \KI{}.

We are left to show depth adjustment,
\begin{align*}
(M, s) \models [\varphi] \E{a}{d}
& \iff (M, s) \models \varphi \implies (M \mid \varphi, (1, s)) \models \E{a}{d} \\
& \iff (M, s) \models \varphi \implies \begin{cases}
	d(a, s) = d + \depth{\varphi} & \quad \text{if } d(a, s) \geq \depth{\varphi} \\
	d(a, s) = d & \quad \text{if } d(a, s) < \depth{\varphi}
	\end{cases} \\
& \iff (M, s) \models \varphi \to \left( (\P{a}{\depth{\varphi}} \wedge \E{a}{d +
	 \depth{\varphi}}) \vee (\neg \P{a}{\depth{\varphi}} \wedge \E{a}{d})
	 \right). \qedhere
\end{align*}
\end{proof}

\section{Complexity}\label{sec:complex}
We first state that adding depth bounds does not change the complexity of
\textbf{S5} and PAL respectively.
\begin{thm}\label{thm:sat}
	The satisfiability problems for~\DBEL{} with \( n \geq 2 \) agents and for
	\TSo{} are \textsf{PSPACE}-complete.
\end{thm}
\begin{proof}
	The lower bound results from \textsf{PSPACE}-completeness of \(
	\textbf{S5}_n \) for \( n \geq 2 \)~\cite{DBLP:journals/ai/HalpernM92}
	and PAL~\cite{DBLP:conf/atal/Lutz06}, respective syntactic fragments
	of~\DBEL{} and~\TSo{}.

	For both logics, we begin by translating \( K_a \varphi \)
	subformulas into \( \P{a}{\depth{\varphi}} \wedge K^\infty_a \varphi \),
	which only increases formula size at most linearly.
	Then, in the case of~\TSo{}, using the same translation
	as Lemma~9 of~\cite{DBLP:conf/atal/Lutz06}, we translate formulas with
	public announcement \( \varphi \) into equivalent formulas \( t(\varphi)
	\) without public announcement such that \( \card{t(\varphi)} \) is at
	most polynomial in \( \card{\varphi} \) (this is possible because the
	axiomatization of \( K^\infty_a \) with relation to public announcements
	is the same).

	We have therefore transformed our formula \( \varphi \) into an
	equivalent formula in the syntactic fragment without \( K_a \) operators
	or public announcements of polynomial size relative to the initial
	formula \( \varphi \)'s size.

We can then use the ELE-World procedure from Figure~6
of~\cite{DBLP:conf/atal/Lutz06} by re-defining types to accommodate for depth
atoms.
As a reminder, we define \( \textbf{cl}(\Gamma) \) for any set of
formulas \( \Gamma \) to be the smallest set of formulas containing \(
\Gamma \) and closed by single negation and sub-formulas.
We then say that \( \gamma \subseteq \textbf{cl}(\Gamma) \) is a type if all of
	the following are true,
\begin{enumerate}
	\item \( \neg \psi \in \gamma \) if and only if \( \psi \not\in \gamma \)
		when \( \psi \) is not a negation
	\item if \( \psi \wedge \chi \in \textbf{cl}(\Gamma) \) then \( \psi
		\wedge \chi \in \gamma \) if and only if \( \psi \in \gamma \)
		and \( \chi \in \gamma \)
	\item if \( K^\infty_a \psi \in \gamma \) then \( \psi \in \gamma \)
	\item\label{item:rule4} if \( \P{a}{d} \in \gamma \) then \( \neg \P{a}{d'} \not\in \gamma
		\) and \( \E{a}{d'} \not\in \gamma \) for all \( d' < d \)
	\item\label{item:rule5} if \( \E{a}{d} \in \gamma \) then \( \E{a}{d'} \not\in \gamma \)
		for all \( d' \neq d \) and \( \neg \P{a}{d'} \not\in \gamma \)
		for \( d' < d \)
	\item\label{item:rule6} if \( \neg \P{a}{d} \in \gamma \) then there
		exists \( d' < d \) such that \( \neg \E{a}{d'} \not\in \gamma \)
	\item\label{item:rule7} \( \neg \P{a}{0} \not\in \gamma \)
\end{enumerate}
Clearly, checking that a subset of \( \textbf{cl}(\Gamma) \) is not a type does
not increase the space complexity of the algorithm.
Lemma~18 from~\cite{DBLP:conf/atal/Lutz06} remains true here, i.e.\ the procedure
ELE-World returns true if and only if the formula is satisfiable.
It is sufficient for this to show that any type has a consistent depth
assignment for all agents, as it is clear that if any of the new rules introduced
for depths are violated the formula is not satisfiable.

If the type contains \( \E{a}{d} \) then it contains only one
such depth atom per rule~\ref{item:rule5}, the only \( \P{a}{d'} \) it
contains are for \( d' \leq d \) per rule~\ref{item:rule4}, and it does not
contain \( \neg \P{a}{d'} \) for \( d' \leq d \) per rule~\ref{item:rule5},
therefore \( d(a) = d \) is a consistent setting.
If it does not contain any \( \E{a}{d} \), it may contain a number of
	inequalities polynomial in \( \card{\varphi} \), that admit a solution in
	\( \N \) by rule~\ref{item:rule7}.
Therefore a possible algorithm is \( d_0 = \max \{ d', \; \P{a}{d'} \in
\gamma \} \) and then \( d(a) = \min \{ d', \; d' \geq d_0, \; \neg \E{a}{d'} \not\in
\gamma \} \).
If no \( \P{a}{d} \) are in the type, then \( d_0 = \min \{ d', \; \neg \P{a}{d'}
\in \gamma \} \) and \( d(a) = \max \{ d', \; d' \leq d_0, \neg \E{a}{d'} \not\in
\gamma\} \) are a possible choice (this choice will always be greater or equal to
\( 0 \) because of rules~\ref{item:rule7} and~\ref{item:rule6} above).
Finally, if there are no depth atoms in the type, the formula is clearly
satisfiable for any choice of \( d(a) \).
\end{proof}

The model checking problem remains \textsf{P}-complete in~\DBEL{},
using the same algorithm as for \textbf{S5}~\cite{DBLP:books/mit/FHMV1995}.
For~\TSo{} and~\TSt{}, the model checking problem is \textsf{P}-complete, as the
same algorithm as PAL can be used, relying on the fact that model size can only
decrease after announcements~\cite{DBLP:conf/aiml/KooiB04} (the lower bounds
results from the fact that PAL is a fragment of both).
This is however not the case of~\DPAL{}, where model size grows after
announcements, potentially exponentially, in fact model checking in~\DPAL{} is
\textsf{NP}-hard~\cite{DBLP:journals/corr/abs-2305-08607}.
\begin{thm}\label{thm:s3exp}
	The complexity of model checking for finite models in~\DPAL{} is
	in \textsf{EXPTIME}. %
	An upper bound in time complexity for checking \( \varphi \) in \( M \)
	is \( O ( 2^{2 \card{\varphi}} {\lVert M \rVert}) \), where \( \lVert M
	\rVert \) is the sum of the number of states and number of pairs in each
	relation of \( M \).
\end{thm}
\begin{proof}
The model-checking algorithm is the same as the one for public
announcement logic~\cite{DBLP:conf/aiml/KooiB04}: a tree is built from
subformulas \( \varphi \), with splits introduced only for subformulas of
the form \( [\psi] \chi \), with \( \psi \) to the left and \( \chi \) to
the right.
Treating a node labeled \( \psi \) means labeling each state in \( M \) with
either \( \psi \) or \( \neg \psi \).
The tree is treated from bottom-left to the top, always going up first except
when a node of the type \( [\psi] \chi \) is found.
In that case, since the nodes in the left sub-tree have been treated, we can
build \( M \mid \psi \) easily in time \( O({\lVert M \rVert}) \) from the
truth value of \( \psi \) and the depth functions of \( M \).
Moreover, the size of \( M \mid \psi \) is at most \( 4 \lVert M \rVert \).

To see this, consider an equivalence class for \( \sim_a \) in \( M \) of size
\( k \), it has exactly \( k^2 \) connections within it.
	The number
of states it creates in \( M \mid \psi \) is at most \( 2k \), and the
number of connections it creates is at most \( 4k^2 \).
Each connection being in exactly one connected component means the bound holds.

Therefore we can recurse in the right sub-tree with \( M \mid \varphi \) to check
\( \chi \) in time \( O( 2^{2 \card{\chi}} \times 4 {\lVert M \rVert} ) \).
	Writing \( O(\lVert M \rVert) \leq c \lVert M \rVert \) the time
	necessary to build \( M \mid \varphi \), we find that checking \( [\psi]
	\chi \) takes time at most \( O( ( c + 2^{2 \card{\psi}} + 2^{2
	\card{\chi} + 2}) {\lVert M \rVert}) = O(2^{2 \card{[\psi] \chi}} {\lVert
	M \rVert}) \).
\end{proof}

\section{Muddy children}\label{sec:muddy}
Consider the well-known muddy children reasoning problem, where \( n \) children
convene after playing outside with mud.
\( k \geq 1 \) of them have mud on their foreheads, but have no way of knowing
it.
The father, an external agent, announces that at least one child has mud on their
forehead.
Then, he repeatedly asks if any child would like to go wash themselves.
After exactly \( k - 1 \) repetitions of the father's question, all muddy
children understand they are muddy and go wash themselves.
Readers unfamiliar with the reasoning problem and its solution are directed
to Van~Ditmarsch et.\ al~\cite{van2007dynamic}'s treatment using PAL\@.

Consider the set of states \( {\{ 0,1 \}}^n \), where each tuple contains \( n \)
entries indicating for each child if they are muddy (\( 1 \)) or not (\( 0 \)).
For the sake of simplicity and since it is of depth \( 0 \), we assume the
father's announcement has taken place and therefore define the Kripke structure
\( \Mn \) with states \( {\{ 0,1 \}}^n \setminus {\{0\}}^n \) with the usual
definition of the agents' knowledge relations~\cite{DBLP:books/mit/FHMV1995}.
We define the \DPAL{} class of muddy children models to be models \( \Mhn \)
extending \( \Mn \) with any depth function.
We name \( m_i \) the atom expressing that child \( i \) is muddy.

We number the agents in \( [| 0; n-1 |] \), where the first \( k \) are muddy,
and focus on the reasoning of one agent (without loss of generality agent \( 0
\)) to understand that it is muddy.
Recall the definition of the dual of public announcements, \(
\langle\varphi\rangle \psi := \neg[\varphi]\neg \psi \) and define the following
series of formulas for \( i \leq k \),
\[
\varphi_i = \langle \neg K_{i-1} m_{i-1} \rangle \langle \neg K_{i-2} m_{i-2}
\rangle \cdots \langle\neg K_1 m_1 \rangle K_0 m_0.
\]
Here \( \varphi_k \) states that if each of the children from \( k-1 \) to \( 1
\) announce one after the other they don't know they are muddy, then child \( 0
\) knows that they (child \(0\)) are muddy~\footnote{These announcements are a
sufficient subset of the full announcements \( \wedge_{j=1,\ldots,n} \neg (K_j m_j
\vee K_j \neg m_j) \) in the usual formulation.}
It is well known this formula is true for unbounded agents in \( \Mn \) in
PAL (it is also a consequence of Theorem~\ref{thm:uppbound} below).
The following two theorems define a sufficient structure of knowledge of depths
for the formula to be true and a necessary condition on the structure of
knowledge of depths for it to be true.
\begin{thm}[Upper bound]\label{thm:uppbound}
    For all three semantics,
	\(
	K_0 \left( \P{0}{k-1} \wedge K_1 ( \P{1}{k-2} \wedge \cdots K_{k-1}(
	\P{k-1}{0}) \cdots ) \right) \to \varphi_k
	\)
	is true in all muddy children models \( \Mhn \) in the initial state.
\end{thm}
Note that this formula directly provides an upper bound on the structure of
depths and knowledge about depths: it shows a sufficient condition on the
knowledge of depths for the problem to be solvable by agent \( 0 \).
Moreover, the upper bound for one child readily generalizes to a sufficient
condition for all children to understand they are muddy: each muddy child must
know they are of depth at least \( k-1 \), know at least some other muddy child knows they
are of depth at least \( k-2 \), and know that that other child knows some other muddy
child knows they are of depth at least \( k-3 \), \textit{etc}.
\begin{proof}
	For the sake of simplicity and since it does not change the treatment of
	the problem, we assume \( n = k \).
	We show the result for \DPAL{}, as the treatments for~\TSo{} and~\TSt{}
	are similar.

	We will show the result by induction over \( k \).
	Denote \( s_k = (1, \ldots, 1) \) the true state of the world where all
	the children are muddy.

	For \( k = 2 \), we assume \( K_0 \P{0}{1} \) and want to show
	\(
		\neg K_1 m_1 \wedge [\neg K_1 m_1] K_0 m_0
	\).
	First notice that \( (\hat{M}_2, s_2) \models \neg K_1 m_1 \), simply because
	it considers the state \( (1, 0) \) to also be possible.
	In the state \( (0, 1) \), child \( 1 \) knows it is muddy.
	Therefore, the set of states for the successful part of the model update
	will be \( (1, (1, 1)) \) and \( (1, (1, 0)) \).
	Moreover, since \( K_0 \P{0}{1} \), it is deep enough in \( s_2 \) to not
	have any links to the unsuccessful part of the model update, therefore it
	knows \( m_0 \).

	Consider some \( k > 2 \), we denote \( S_i \) the set of states that
	are ``active'' when considering \( \varphi_i \).
	More precisely, we set \( S_i = {\{0,1\}}^i \times {\{ 1 \}}^{k-i}
	\setminus {\{ 0 \}}^k \).
	We will show that after \( k-i \) announcements, the remainder of the
	problem is equivalent to checking \( \varphi_i \) on the subgraph induced
	by the states \( S_i \).
	This is evident for \( i = k \) by definition, we now show by descending
	induction that it is equivalent to checking \( \varphi_2 \) on \( S_2 \),
	which we have just verified to be true.

	Firstly, it is true that \( (\Mhn, s_k) \models \neg K_{k-1} m_{k-1} \)
	since child \( k-1 \) considers possible the state \( (1, \ldots, 1, 0)
	\).
	The set of states in which \( K_{k-1} m_{k-1} \) holds is exactly \( (0,
	\ldots, 0, 1) \).
	Therefore, the model update will create a copy of all other states.
	We then notice that the set of states whose last component is \( 0 \) can
	be ignored in the rest of the problem: they are not reachable from \( s_k
	\) by any sequence of \( \sim_i \) that does not contain \( \sim_{k-1}
	\) and the rest of the formula \( \varphi_{k-1} \) to be checked does
	not use any modal operators for agent \( k-1 \) any more.
	These states will never be reached and can therefore be
	removed without altering the result of the rest of the execution.

	We are therefore restricting ourselves, after the model update, to the
	set of states \( S_{k-1} \) in the positive part of the model.
	Note however there are still possibly links between the negative part of
	the model and \( S_{k-1} \) in the positive part of the model.
	We will show that these links have no effect on the checking of the rest
	of the formula, by showing that links for child \( i \) find themselves
	in \( S_{k-1} \setminus S_i \): therefore, by the time we query modal
	operator \( i \), the set of ignored states will contain all states with
	a link for child \( i \).

	For child \( i < k-1 \), the information we have about its depth is \(
	K_0 K_1 \cdots K_i \P{i}{k-1-i} \) before the model update.
	Therefore, we in particular know it is deep enough for the announcement
	(which is of depth \( 1 \leq k-1-i \)) in the set of states in which the
	\( i \) first components might have changed compared to \( s_k \) but the
	last \( k-1-i \) are all fixed to \( 1 \): this is exactly \( S_i \).

	We have shown that the recursive check in \( M \mid \neg
	K_{k-1} m_{k-1} \) will take place on a set of states for which the
	execution is equivalent to \( S_{k-1} \) and on which we will have to
	check the formula \( \varphi_{k-1} \).
	Finally, since the depths of each agent other than \( k-1 \) was at
	least \( 1 \) on \( S_{k-2} \), they are reduced by \( 1 \) and the
	induction hypothesis on depths for \( k-2 \) is also verified.
\end{proof}

\begin{thm}[Lower bound]\label{thm:lowbound}
	For~\DPAL{}, the formula
	\( \varphi_k \to K_0 \P{0}{k-1} \)
	is true in all models \( \Mhn \).
\end{thm}
\begin{proof}

	We use the notations from the proof of Theorem~\ref{thm:uppbound} above.
	Notice first that all of the announcements remain true when they are
	performed, because \( \neg K^\infty_{k-1} m_{k-1} \to \neg K_{k-1}
	m_{k-1} \) and the implicant is true by the usual lower bound for muddy
	children (it takes \( k \) announcements for any child to know they are
	muddy).

	Assume by contraposition that \( d(0, s_k) = i < k-1 \) or \( d(0,
	\tilde{s}_k) = i < k-1 \) initially, where \( \tilde{s}_k \) is the
	state \( (0, 1, \ldots, 1) \) of \( \Mhn \).
	After \( i \) public announcements, it will be true that \( \neg K_0 m_0
	\) still, as well as \( \neg K_0 \neg \E{0}{0} \) since each public
	announcement is of depth \( 1 \).
	The former is a consequence of the usual lower bound for muddy children,
	and can be derived from the proof in Theorem~\ref{thm:uppbound}
	using symmetry between \( 0 \) and \( k-1-i \) after the \( i \)
	announcements and monotonicity of knowledge of atoms: if the depths are
	lower than they were in the previous proof, there are more states and
	more links in the updated model and therefore \( \neg K_{k-1-i} m_{k-1-i}
	\) remains true.

	Therefore in this model after \( i \) announcements, either \( s_k \) or
	\( \tilde{s}_k \) sees agent \( 0 \) of depth \( 0 \) and both states
	are still connected by \( \sim_0 \).
	This means that for the next announcement, since \( \neg K_0 m_0 \) after
	each announcement except potentially the last using the same argument as
	above, we will have the chain of connections \( (1, s_k) \sim'_0 (0, s_k)
	\sim'_0 (0, s_k') \) or \( (1, s_k) \sim'_0 (1, \tilde{s}_k) \sim'_0 (0,
	\tilde{s}_k) \).
	This means that by an immediate induction, after the \( k-i \)
	announcements it is still true that \( \neg K_0 m_0 \): this is a
	contradiction with \( \varphi_k \).
\end{proof}
A stronger lower bound for each child is
available~\cite{DBLP:journals/corr/abs-2305-08607}, with recursive conditions on
the depth of all agents similarly to Theorem~\ref{thm:uppbound}.
This formula provides a lower bound on the knowledge of depths of the agent \( 0 \) to be able to solve the problem: it must be depth at least \( k-1 \) and know so.
By symmetry, this generalizes to any child or any set of children solving the problem.

Finally, we present propositions that illustrate how \textit{amnesia} in~\TSo{}
(Proposition~\ref{prop:amnesia}) and \textit{knowledge leakage} in~\TSt{}
(Proposition~\ref{prop:leakage}) manifest in the muddy children problem.
These propositions are easily verified by computing explicitly the models after updates.
\begin{prop}[Amnesia in~\TSo{}]
	Consider the instance of muddy children \( M_3 \), where child \( i
	\) is unambiguously of depth \( 2-i \), i.e.\ \( d(i, \cdot) = 2-i \).
	The formula
	\(
		\langle \neg K_2 m_2 \rangle \langle \neg K_1 m_1 \rangle \neg
		K_2 \top
	\)
	is true in \TSo{} but not in~\DPAL{} or~\TSt{}.
	This means that in~\TSo{}, after the first two announcements, agent \( 2 \)
	does not know anything anymore.

 \end{prop}
 \begin{prop}[Knowledge leakage in~\TSt{}]\label{prop:badmuddy}
	The formula
	\( \langle K_1 \neg K_2 m_2 \rangle K_1 K_0 m_0 \)
    is true in~\TSt{} but not in~\DPAL{} or~\TSo{}.
	In~\TSt{}, agent \( 1 \) has deduced the conclusion of agent \( 0
	\)'s reasoning, despite not being deep enough to perceive the
	announcement.
	Moreover, if agent \( 0 \) were of depth \( 1 \) it would not be true
	that \( \langle K_1 \neg K_2 m_2 \rangle K_0 m_0 \): agent \( 0 \) would
	not be able to deduce what agent \( 1 \) has deduced.
\end{prop}

\paragraph{Library}
Alongside this paper, we publish code for a library for multi-agent epistemic
logic model checking and visualization in Python.
It implements depth-unbounded PAL models as well as~\DPAL{}, \TSo{} and~\TSt{}.
The code is available
\href{https://gitlab.com/farid-fari/depth-bounded-epistemic-logic}{in an online
repository}~\cite{bworld}. %
The code can also be used to generate illustrations of model updates in the muddy
children reasoning problem~\cite{DBLP:journals/corr/abs-2305-08607}
under the assumptions of Theorem~\ref{thm:uppbound} above.

\paragraph{Conclusion}
We have shown how \textbf{S5} and public announcement logic (PAL) can be extended
to incorporate bounded-depth agents.
We have shown completeness results for several of the resulting
logics and explored the relationship between public announcements and knowledge
in~\DPAL{}, as well as complexity bounds for these logics.
We finally illustrated the behavior of depth-bounded agents in the muddy children
reasoning problem, where we showed upper and lower bounds on depths (and
recursive knowledge of depths) necessary and sufficient to solve the problem.
These results extend epistemic logics to support formal reasoning about agents
with limited modal depth.

\nocite{*}
\bibliographystyle{eptcs}
\bibliography{farid}

\appendix{}
\section{Axiomatization proofs}\label{app:ax}
\begin{table}[h]
\centering
 \begin{tabular}{|| r | l ||}
 \hline
 All propositional tautologies & \( p \to p \), \textit{etc}. \\
 Deduction & \( (K^\infty_a \varphi \wedge K^\infty_a (\varphi \to \psi)) \to
 K^\infty_a \psi \) \\
 \hline
 Truth & \( K^\infty_a \varphi \to \varphi \) \\
 Positive introspection & \( K^\infty_a \varphi \to K^\infty_a K^\infty_a \varphi \) \\
 Negative introspection & \( \neg K^\infty_a \varphi \to K^\infty_a \neg K^\infty_a \varphi \) \\
 \hline
 Depth monotonicity & \( \P{a}{d} \to \P{a}{d-1} \) \\
 Exact depths & \( \P{a}{d} \leftrightarrow \neg (E_a^0 \vee \cdots \vee E_a^{d-1})  \) \\
 Unique depth & \( \neg(E_a^{d_1} \wedge E_a^{d_2}) \) for \( d_1 \neq d_2 \) \\
 \hline
 Bounded knowledge & \( K_a \varphi \leftrightarrow \P{a}{\depth{\varphi}} \wedge K^\infty_a
 \varphi \) \\
 \hline
 \hline
 \textit{Modus ponens} & From \( \varphi \) and \( \varphi \to \psi \), deduce \(
 \psi \) \\
 Necessitation & From \( \varphi \) deduce \( K^\infty_a \varphi \) \\
 \hline
 \end{tabular}
	\caption{Sound and complete axiomatization of \DBEL{} over \( \H^\infty \).}\label{tab:axb2}
\end{table}
\begin{thm}\label{thm:altaxb}
	Axiomatization from Table~\ref{tab:axb2} is sound and complete with
	respect to \DBEL{} over \( \H^\infty \).
\end{thm}
\begin{proof}
	Soundness of all of these axioms is immediate: the definition of \(
	K^\infty_a \) follows that of \textbf{S5} and so do the axioms, those
	concerning depth atoms are consequences of linear arithmetic, and the
	bounded knowledge axiom follows immediately from the definition of \( K_a
	\) in the semantics.

	For completeness, first note we can translate any formula \( \varphi \) in
	\( \H^\infty \) into an equivalent formula \( t(\varphi) \) that does not
	contain any \( \P{a}{d} \) atoms or \( K_a \) modal operators using the
	exact depths and bounded knowledge axioms (which we know to be sound).
	Call \textbf{S5D} this fragment of \textbf{DBEL}.

	We will use a proof through the \textsc{Lindenbaum} lemma and the truth
	lemma, to this end we need to complete the definition for the canonical
	model to add a depth function.
	As a reminder, the proof is as follows: if \( \varphi \) cannot be shown
	within the axiomatization in Table~\ref{tab:axb}, i.e. \( \not\vdash
	\varphi \), then we show that \( \not\models \varphi \) by showing there
	is a state in the canonical model in which it does not hold.

	The canonical model \( M^c \) is the model whose states are maximally
	consistent sets \( \Gamma \) of formulas for our axiomatization and whose
	states
	are related by \( \sim_a \) if the set of formulas \( a \) knows is the
	same in both states.
	Its valuation function for atoms \( V(\Gamma) \) is simply the set of
	axioms in \( \Gamma \), i.e. \( \Gamma \cap \atoms \).

	We restrict \( M^c \) to sets \( \Gamma \) that contain at least some \(
	\E{a}{d} \) for each agent \( a \in \agents \) and by the unique depth
	axiom we define \( d(a, \Gamma) = \max \{ d, \; \E{a}{d} \in \Gamma \}
	\), since \( \Gamma \) contains exactly one depth to be consistent.
	This completes \( M^c \) into a \DBEL{} model.

	Since \( \not\vdash \varphi \), the set \( \{ \neg \varphi \} \) is
	consistent for our axiomatization.
	We must now show we can extend this set into a maximal consistent set of
	formulas that contains a depth atom \( \E{a}{d} \) for each agent \( a
	\).

	However, this stronger requirement is not satisfied by the usual
	\textsc{Lindenbaum} lemma, since a consistent set of formulas could be \(
	\{ \P{a}{d}, \; d \in \N \} \) (which is not consistent with any \(
	\E{a}{d} \)).
	Note however we only need it to hold for a finite set of formulas (namely
	\( \{ \neg \varphi \} \)): Lemma~\ref{lmm:lindapp} below proves this
	version of the \textsc{Lindenbaum} lemma, by showing there must exist
	some \( \E{a}{d} \) that is consistent with any finite set for each \( a
	\), and then a maximally consistent set can be derived using the
	traditional \textsc{Lindenbaum} lemma.

	Finally, the truth lemma shows that \( \varphi \in \Gamma \iff (M^c,
	\Gamma) \models \varphi \) by induction on \( \varphi \) and is enough
	to conclude (since the maximal consistent set containing \( \neg \varphi
	\) will not verify \( \varphi \)).
	Most induction cases are the same as for \textbf{S5}, the only new
	symbols left in our formula \( \varphi \) are the \( \E{a}{d} \) atoms,
	and the truth lemma is immediately true for them by definition of the
	depth function of \( M^c \).

	Finally, if \( \models \varphi \), then \( \models t(\varphi) \) by the
	soundness of the axiomatization and definition of the transformation,
	then \( \textbf{S5D} \vdash t(\varphi) \) since we have just shown the
	completeness of this fragment.
	Finally, this must mean \( \DBEL \vdash t(\varphi) \) and then
	\( \vdash \varphi \) since the transformations of \( t \) can be
	performed using equivalences in our axiomatization: we have shown
	completeness.
\end{proof}

\begin{table}
\centering
 \begin{tabular}{|| r | l ||}
 \hline
 All propositional tautologies & \( p \to p \), \textit{etc}. \\
 Deduction & \( (K^\infty_a \varphi \wedge K^\infty_a (\varphi \to \psi)) \to
 K^\infty_a \psi \) \\
 \hline
 Truth & \( K^\infty_a \varphi \to \varphi \) \\
 Positive introspection & \( K^\infty_a \varphi \to K^\infty_a K^\infty_a \varphi \) \\
 Negative introspection & \( \neg K^\infty_a \varphi \to K^\infty_a \neg K^\infty_a \varphi \) \\
 \hline
 Atomic permanence & \( [\varphi] p \leftrightarrow \varphi \to p \) \\
 Depth adjustment & \( \forall d \in \Z, \; [\varphi] \E{a}{d} \leftrightarrow \left(
	 \varphi \to \E{a}{\depth{\varphi} + d} \right) \) \\
 Negation announcement & \( [\varphi] \neg \psi \leftrightarrow (\varphi \to \neg
	 [\varphi] \psi) \) \\
 Conjunction announcement & \( [\varphi] (\psi \wedge \chi) \leftrightarrow
	 ([\varphi] \psi \wedge [\varphi] \chi) \) \\
 Knowledge announcement & \( [\varphi] K^\infty_a \psi \leftrightarrow
	 (\varphi \to K^\infty_a [\varphi] \psi ) \) \\
 Announcement composition & \( [\varphi] [\psi] \chi \leftrightarrow
	 ([\varphi \wedge [\varphi] \psi] \chi) \) \\
 \hline
 Depth monotonicity & \( \P{a}{d} \to \P{a}{d-1} \) \\
 Exact depths & \( \P{a}{d} \leftrightarrow \neg (E_a^0 \vee \cdots \vee E_a^{d-1})  \) \\
 Unique depth & \( \neg(E_a^{d_1} \wedge E_a^{d_2}) \) for \( d_1 \neq d_2 \) \\
 \hline
 Bounded knowledge & \( K_a \varphi \leftrightarrow \P{a}{\depth{\varphi}} \wedge K^\infty_a
 \varphi \) \\
 \hline
 \hline
 \textit{Modus ponens} & From \( \varphi \) and \( \varphi \to \psi \), deduce \(
 \psi \) \\
 Necessitation & From \( \varphi \) deduce \( K^\infty_a \varphi \) \\
 \hline
 \end{tabular}
\caption{Sound and complete axiomatization of~\TSo{}.}\label{tab:ax4}
\end{table}

\begin{thm}\label{thm:ax4}
	The axiomatization in Table~\ref{tab:ax4} is sound and complete with
	respect to~\TSo{}.
\end{thm}
\begin{proof}
	Soundness of the axioms of~\DBEL{} is proven in Theorem~\ref{thm:altaxb}.
	Soundness of all axioms for public announcement is also a consequence of
	their definition in PAL with which they share their definition, except
	for depth adjustment for which the proof is relatively immediate.

	For completeness, we translate any formula \( \varphi \) into
	\( t(\varphi) \) by removing public announcements, \( K_a \) modal
	operators and \( \P{a}{d} \) atoms by using the sound axioms
	from Table~\ref{tab:ax4}.
	The formula \( t(\varphi) \) is in the syntactic fragment
	\textbf{S5D}, thus we can use completeness shown
	in Theorem~\ref{thm:axb} to show \( \vdash t(\varphi) \), which implies
	\( \vdash \varphi \) within the axiomatization of Table~\ref{tab:ax4} by
	using the same axioms in the opposite direction.
\end{proof}

\section{\textsc{Lindenbaum} lemma with depth assignments}\label{app:lind}
\begin{lmm}\label{lmm:lindapp}
For every agent \( a \) and finite consistent set of formulas \( \Gamma \)
without public announcement, \( \P{b}{d} \) literals or \( K_b \) operators for
all \( b \), there exists some \( d \in \N \) such that \( \Gamma \cup \{
\E{a}{d} \} \) is a consistent set.
\end{lmm}
\begin{proof}
	Fix agent \( a \).
	As \( \Gamma \) is a finite set of finite formulas, the set of exact
	depth atoms for \( a \) that appear in its formulas is included in a
	finite set \( F = \{ \E{a}{0}, \ldots, \E{a}{D} \} \) for some \( D \in
	\N \).

	We can add to \( \Gamma \) instances of the unique depth axiom for each
	pair of integers in \( [| 0; D |] \) while maintaining consistency.
	The set \( \Gamma \) can then be seen as a consistent set of formulas for
	\textbf{S5} over the set of atoms \( F \cup \atoms \), i.e.\ consistent
	in the axiomatization of Table~\ref{tab:axb2} without depth axioms or
	bounded knowledge (or tautologies involving symbols not in the language
	of \textbf{S5}).
	Therefore there is an \textbf{S5} model \( (M, s) \) that satisfies it
	by the usual \textsc{Lindenbaum} lemma and the truth lemma (the canonical
	model here).

	In \( (M, s) \), if any of the \( \E{a}{d} \) are valued to \( \top \),
	then at most one of them is satisfied (since we added the unique depth
	axiom for all pair of depths).
	If all of the \( \E{a}{d} \) are valued to \( \bot \), then we can
	introduce a new atom \( \E{a}{D+1} \) and set its value to \( \top \) in
	all states of the model.
	All of the unique depths axioms for \( D+1 \) and \( d \leq D \) can be
	added to \( \Gamma \) without making it inconsistent.

	In both cases, let \( d_0 \) be the value of the unique \( \E{a}{d_0} \)
	valued to \( \top \) in this final model.
	We claim that \( \{ \varphi, \E{a}{d_0} \} \) must be a consistent set.
	Indeed, a proof of its inconsistency with the axioms from
	Table~\ref{tab:axb2} must only involve axioms from \textbf{S5} and unique
	depths axioms for the set \( F \), since none of the symbols \(
	\P{a}{d} \) or \( K_a \) are necessary in a proof (they can be replaced
	by their equivalents with \( \E{a}{d} \) and \( K^\infty_a \) without
	changing the conclusion) and any occurrence of \( \E{a}{d} \) for \( d >
	D+1 \) can be replaced by \( \bot \) while maintaining the truthfulness
	and conclusion of the proof.

	Therefore, such an inconsistency proof would also hold within
	\textbf{S5}, which is a contradiction with soundness since these
	formulas are verified in a consistent set (the set of true formulas in \(
	(M, s) \)).
\end{proof}

\section{Proofs for Section~\ref{sec:dpal}}\label{app:props}
\begin{prop}\label{prop:col3app}
	Formulas~\KI{} and~\TAz{} are valid for \DPAL{} in the unambiguous depths
	setting.
\end{prop}
\begin{proof}
To prove~\KI{}, suppose without loss of generality that \( (M, s) \models \neg
\P{a}{\depth{\varphi}} \wedge \varphi \).
In particular, this means that in \( M \mid \varphi \), we have \( (0, s) \sim'_a
(1, s) \) and therefore the equivalence class of \( (1, s) \) in \( M \mid
\varphi \) contains all \( (0, s') \) whenever \( s' \sim_a s \).
Then,
\begin{align}\label{eq:rewrLHS}
(M, s) \models [\varphi] K_a \psi
& \iff (M, s) \models \varphi \implies (M \mid \varphi, (1, s)) \models K_a \psi
	\notag \\
& \iff (M \mid \varphi, (1, s)) \models \P{a}{\depth{\psi}} \text{ and } \forall
	s', j, (j, s') \sim'_a (1, s) \implies (M \mid \varphi, (j, s')) \models
	\psi \notag \\
& \iff (M \mid \varphi, (1, s)) \models \P{a}{\depth{\psi}} \text{ and } \forall
	s' \sim_a s, \;
\begin{cases}
(M \mid \varphi, (0, s')) \models \psi \\
(M, s') \models \varphi \implies (M \mid \varphi, (1, s')) \models \psi.
\end{cases}
\end{align}

On the other hand,
\begin{equation}\label{eq:rewrRHS}
(M, s) \models K_a \psi
\iff (M, s) \models \P{a}{\depth{\psi}} \text{ and } \forall s', s' \sim_a
s \implies (M, s') \models \psi.
\end{equation}

We prove by structural induction over \( \psi \in \Ha \) the stronger
equivalence,
\begin{equation}\label{eq:ihki}
\forall s' \sim_a s, \quad
\begin{cases}
	(M, s') \models \psi \iff (M \mid \varphi, (0, s')) \models \psi \\
	(M, s') \models \varphi \implies \left( (M, s') \models \psi \iff
	(M \mid \varphi, (1, s')) \models \psi \right).
\end{cases}
\end{equation}
Given that \( (M, s) \not\models \Paphi \), we have \( (M \mid \varphi, (1, s))
\models \P{a}{\d{\psi}} \iff (M, s) \models \P{a}{\d{\psi}} \).
Therefore the depth conditions in equations~\eqref{eq:rewrLHS}
and~\eqref{eq:rewrRHS} are the same and since both sides are true if \( (M, s)
	\not\models \varphi \), equation~\eqref{eq:ihki} is enough to prove~\KI{}.

For \( \psi \in \atoms \), it is true because \( V'((j, s')) = V(s') \) for all
\( j \) and \( s' \) (note that \( \atoms \) does not include depth atoms).
For depth atoms about \( a \), it is a consequence of \( (M, s) \models K_a \neg
	\Paphi \) by the depth unambiguity condition~\eqref{eq:unam}, which means
	the depth of \( a \) is unchanged in all \( s' \sim_a s \) after the
	model update.

The cases where \( \psi = \psi_1 \wedge \psi_2 \) and \( \psi = \neg \chi \) are
immediate, by the way these operators coincide with the usual propositional logic
definition on both sides of the equivalences.

If \( \psi = K_a \chi \) and \( s' \sim_a s \), recall that by the depth
unambiguity condition~\eqref{eq:unam} we have \( (M, s') \models \neg \Paphi \).
 Therefore, if \( (M, s') \models \varphi \),
\begin{align*}
(M \mid \varphi, (1, s')) \models \psi
& \iff d(a, s') \geq \d{\chi} \; \text{ and } \; \forall (j, s'') \sim'_a (1, s'), \;
	(M \mid \varphi, (j, s'')) \models \chi \\
& \iff d(a, s') \geq \d{\chi} \; \text{ and } \; \forall s'' \sim_a s, \; \begin{cases}
	(M \mid \varphi, (0, s'')) \models \chi \\
	(M, s'') \models \varphi \implies (M \mid \varphi, (1, s'')) \models \chi
	\end{cases} \\
& \iff d(a, s') \geq \d{\chi} \; \text{ and } \; \forall s'' \sim_a s', \;
	(M, s'') \models \chi \\
& \iff (M, s') \models \psi,
\end{align*}
where we have used the induction hypothesis~\eqref{eq:ihki} for \( \chi \) once
in each direction.
The first equivalence in equation~\eqref{eq:ihki} is even easier to verify, by
the same technique.
The case for \( \psi = K^\infty_a \chi \) is directly implied by this proof, as
there are no depth conditions to verify.

To prove public announcements, we will need a stronger induction hypothesis than~\eqref{eq:ihki}.
Write for any \( s \), \( 1_0(s) = s \) and \( 1_n(s) = (1, 1_{n-1}(s)) = (1, \ldots, (1, s)) \).
We posit,
\begin{align}
& \forall n \in \N, \; \forall \psi_1, \ldots, \psi_n, \; \forall s' \sim_a s, \;
(M, s') \models \P{a}{\d{\psi_1} + \cdots \d{\psi_n} + \d{\psi}} \text{ and }
(M, s') \models \neg \Paphi \implies \notag \\
& (M, s') \models \psi_1 \text{ and } 
(M \mid \psi_1, (1, s')) \models \psi_2 \text{ and }
\ldots 
\text{ and } 
(M \mid \psi_1 \mid \cdots \mid \psi_{n-1}, 1_{n-1}(s')) \models \psi_n \implies \notag \\
&
\begin{cases}
	(M \mid \psi_1 \mid \cdots \mid \psi_n, 1_n(s')) \models \psi \iff (M \mid \varphi  \mid \psi_1 \mid \cdots \mid \psi_n, 1_n((0, s'))) \models \psi \\
	(M, s') \models \varphi \implies \left( (M \mid \psi_1 \mid \cdots \mid \psi_n, 1_n(s')) \models \psi \iff
	(M \mid \varphi \mid \psi_1 \mid \cdots \mid \psi_n, 1_n((1, s'))) \models \psi \right).
\end{cases}\label{eq:ihkistr}
\end{align}
Note we slightly abuse notation here and some of these states might not exist,
the convention is that the equivalences need only hold when the states exist in
the models on both sides.
The implicant implies that the left-hand term always exists.

Taking this for \( n = 0 \) is sufficient to conclude on~\KI{}, since both equations~\eqref{eq:rewrLHS} and~\eqref{eq:rewrRHS} will be false whenever \( (M, s) \not\models \P{a}{\d{\psi}} \).

The cases for atoms, negations and conjunction are clear for the same reasons as they were in equation~\eqref{eq:ihki}.
The case for depth atoms for \( a \) is direct, since the assumption \( (M, s')
\models \P{a}{\d{\psi_1} + \cdots \d{\psi_n}} \) implies that the depth of \( a
\) after the \( \psi_1, \ldots, \psi_n \) announcements is its initial depth
minus the sum of the depths of all the announcements, and the assumption that it
is not deep enough for \( \varphi \) means its depth does not change with the
announcement of \( \varphi \).

The case for modal operators \( K_a \) relies on the fact that depth atoms are
preserved (by the induction hypothesis for depth atoms) and the relations verify
in \( M \mid \psi_1 \mid \cdots \mid \psi_n \) when these states exist,
\[
(1, (1, \ldots (1, s_1))) \sim^{(n)}_a (1, (1, \ldots (1, s_2)))
\iff
s_1 \sim_a s_2,
\]
by denoting \( \sim^{(k)}_a \) the relation for \( a \) in a model after \( k \)
announcements.
And similarly in \( M \mid \varphi \mid \psi_1 \mid \cdots \mid \psi_n \),
\[
(1, (1, \ldots (j, s_1))) \sim^{(n+1)}_a (1, (1, \ldots (k, s_2)))
\iff
(j, s_1) \sim'_a (k, s_2)
\iff
s_1 \sim_a s_2.
\]
This also implies the case for \( K^\infty_a \), since the verification is the same without the depth condition.

Finally, if \( \psi = [\psi'] \chi \), we verify that for \( s' \sim_a s \) such that \( (M, s') \models \varphi \),
\begin{align}
& (M \mid \psi_1 \mid \cdots \mid \psi_n, 1_n(s')) \models \psi \notag \\
& \iff (M \mid \psi_1 \mid \cdots \mid \psi_n, 1_n(s')) \models \psi' \implies (M \mid \psi_1 \mid \cdots \mid \psi_n \mid \psi', 1_{n+1}(s')) \models \chi \notag \\
& \iff (M \mid \psi_1 \mid \cdots \mid \psi_n, 1_n(s')) \models \psi' \implies (M \mid \varphi \mid \psi_1 \mid \cdots \mid \psi_n \mid \psi', 1_{n+1}((1, s'))) \models \chi \notag \\
& \iff (M \mid \varphi \mid \psi_1 \mid \cdots \mid \psi_n, 1_n((1, s'))) \models \psi' \implies (M \mid \varphi \mid \psi_1 \mid \cdots \mid \psi_n \mid \psi', 1_{n+1}((1, s'))) \models \chi \notag \\
& \iff (M \mid \varphi \mid \psi_1 \mid \cdots \mid \psi_n, 1_n((1, s'))) \models \psi
	\label{eq:kipa}
\end{align}
when the latter state exists.
Our first use of the induction hypothesis on \( \chi \) is justified because the
left-hand side of the implication is the \( n+1 \) term in the assumptions for
the induction hypothesis in equation~\eqref{eq:ihkistr} (and \( \d{\psi} =
\d{\psi'} + \d{\chi} \)).
The second use of the induction hypothesis on \( \psi' \) is justified for the
same depth reason and the other assumptions remain the same.
Once more the case for \( (0, s') \) is very similar.

For~\TAz{}, we assume without loss of generality that \( (M, s) \models K_a(
\P{a}{\depth{\varphi}}) \wedge \varphi \) (using the depth unambiguity
condition~\eqref{eq:unam}), this means in particular
the equivalence class of \( (1, s) \) in \( M \mid \varphi \) is \( \{ (1, s'),
\; s' \sim_a s, \; (M, s') \models \varphi \} \) since no state
equivalent to \( s \) by \( \sim_a \) has \( a \) not deep enough for \( \varphi
\).
Using the same reasoning as in equation~\eqref{eq:rewrLHS}, we have,
\[
(M, s) \models [\varphi] K_a \psi \iff
(M \mid \varphi, (1, s)) \models \P{a}{\d{\psi}} \text{ and }
\forall s' \sim_a s, \; (M, s') \models \varphi \implies (M \mid \varphi, (1, s')) \models \psi.
\]
Moreover, we have,
\begin{equation}\label{eq:rewrta}
(M, s) \models K_a [\varphi] \psi \iff
(M, s) \models \P{a}{\d{\varphi} + \d{\psi}} \text{ and }
\forall s' \sim_a s, \;
(M, s') \models \varphi \implies (M \mid \varphi, (1, s')) \models \psi.
\end{equation}
Since \( (M, s) \models \Paphi \), the depth of \( a \) in \( (M \mid \varphi,
(1, s)) \) is its depth in \( (M, s) \) minus \( \d{\varphi} \).
This means that
\[ (M \mid \varphi, (1, s)) \models \P{a}{\d{\psi}} \iff (M, s)
\models \P{a}{\d{\varphi} + \d{\psi}}. \qedhere \]
\end{proof}

\begin{prop}\label{prop:propsstrapp}
	\DPAL{} verifies~\KIp{} and~\TA{}.
\end{prop}
\begin{proof}
For~\KIp{}, in light of equations~\eqref{eq:rewrLHS} and~\eqref{eq:rewrRHS}, we
use the following induction hypothesis,
\begin{align}\label{eq:ihkip}
	& \forall s, a, \quad (M, s) \models K^\infty_a \F(\psi) \implies \notag \\
	& \forall s' \sim_a s, \quad \begin{cases}
(M \mid \varphi, (0, s')) \models \psi
\iff
(M, s') \models \psi \\
(M, s') \models \varphi \implies
\left( (M \mid \varphi, (1, s')) \models \psi
\iff
(M, s') \models \psi \right).
\end{cases}
\end{align}
Assume that \( (M,s) \models \varphi \wedge \F(K_a \psi) \).
In particular, \( (M, s) \models \neg K^\infty_a(\varphi \to \Paphi) \).
First notice that this condition allows us to write, \( (0, s') \sim'_a (1, s)
\iff s' \sim_a s \).
Indeed, since there exists some \( s'' \sim_a s \) where \( a \) is of depth
strictly less than \( \d{\varphi} \) and \( \varphi \) holds, we deduce the
chain of connections, \( (1, s) \sim'_a (1, s'') \sim'_a (0, s'') \sim'_a (0, s')
\) for any \( s' \sim_a s \) (and the direct implication is immediate).

Moreover, we have assumed \( (M, s) \models \varphi \wedge (\varphi \to \neg
\Paphi \vee \P{a}{\d{\varphi} + \d{\psi}}) \).
In either case of the disjunction, the depth conditions of
equations~\eqref{eq:rewrLHS} and~\eqref{eq:rewrRHS} become equivalent as they did
in the proof of~\KI{}.
Therefore, proving the induction hypothesis~\eqref{eq:ihkip} is sufficient to
conclude~\KIp{} here.

The cases for atoms, negations and conjunctions is the same as in the proof of
Proposition~\ref{prop:col3app}, as the induction hypothesis holds because
\( \F(\neg \psi) = \F(\psi) \), \( \F(\psi_1 \wedge \psi_2) = \F( \psi_1)
\wedge \F(\psi_2) \), and by commutativity of \( K^\infty_a \) with conjunction.

If \( \psi = K_b \chi \) for some agent \( b \in \agents \), for some fixed \( s'
\sim_a s \), we know that \( (M, s') \models \neg K^\infty_b( \varphi \to
\P{b}{\d{\varphi}} ) \) as well as \( (M, s') \models K^\infty_b \F(\chi) \).
Moreover, the condition \( (M,s') \models K^\infty_b( \varphi \to \neg
\P{b}{\d{\varphi}} \vee \P{b}{\d{\varphi} + \d{\chi}}) \) implies that the depth
of \( b \) will be greater or equal to \( \d{\chi} \) in \( (M \mid \varphi, (1,
s')) \) if and only if it was in \( (M, s') \).
If \( (M, s') \models \varphi \), by once more using the induction
hypothesis~\eqref{eq:ihkip} for \( b \) in \( s' \), we obtain that,
\begin{align*}
(M \mid \varphi, (1, s')) \models \psi
& \iff d(b, s') \geq \d{\chi} \; \text{ and } \; \forall (j, s'') \sim'_b (1, s'), \;
	(M \mid \varphi, (j, s'')) \models \chi \\
& \iff d(b, s') \geq \d{\chi} \; \text{ and } \; \forall s'' \sim_b s', \; \begin{cases}
	(M \mid \varphi, (0, s'')) \models \chi \\
	(M, s'') \models \varphi \implies (M \mid \varphi, (1, s'')) \models \chi
	\end{cases} \\
& \iff d(b, s') \geq \d{\chi} \; \text{ and } \; \forall s'' \sim_b s', \;
	(M, s'') \models \chi \\
& \iff (M, s') \models \psi.
\end{align*}
The case for \( (0, s') \) is the same, since its equivalence class in \( M \mid
\varphi \) is the same and the depth condition is the same.
The case for \( \psi = K^\infty_b \chi \) is implied by this proof, as
there are no depth conditions to verify.

Finally, checking public announcements involves performing the same induction
hypothesis strengthening as in the proof of~\KI{} in its
equation~\eqref{eq:ihkistr}.
The new induction hypothesis becomes,
\begin{align}
& \forall s, a, \quad (M, s) \models K^\infty_a \F(\psi) \implies \notag \\
& \forall n \in \N, \; \forall \psi_1, \ldots, \psi_n, \; \forall s' \sim_a s, \;
(M, s') \models \P{a}{\d{\psi_1} + \cdots \d{\psi_n} + \d{\psi}} \text{ and }
(M, s') \models \neg \Paphi \implies \notag \\
& (M, s') \models \psi_1 \text{ and } 
(M \mid \psi_1, (1, s')) \models \psi_2 \text{ and }
\ldots 
\text{ and } 
(M \mid \psi_1 \mid \cdots \mid \psi_{n-1}, 1_{n-1}(s')) \models \psi_n \implies \notag \\
&
\begin{cases}
	(M \mid \psi_1 \mid \cdots \mid \psi_n, 1_n(s')) \models \psi \iff (M \mid \varphi  \mid \psi_1 \mid \cdots \mid \psi_n, 1_n((0, s'))) \models \psi \\
	(M, s') \models \varphi \implies \left( (M \mid \psi_1 \mid \cdots \mid \psi_n, 1_n(s')) \models \psi \iff
	(M \mid \varphi \mid \psi_1 \mid \cdots \mid \psi_n, 1_n((1, s'))) \models \psi \right).
\end{cases}\notag %
\end{align}
Note we slightly abuse notation here and some of these states might not exist,
the convention is that the equivalences need only hold when the states exist in
the models on both sides.
The implicant implies that the left-hand term always exists.

Checking atoms, depth atoms, negation and conjunction is the same as in the proof of~\KI{} once more.
Checking modal operators \( K_a \) and \( K^\infty_a \) is similar to the proof of~\KI{} using induction hypothesis~\eqref{eq:ihkistr}, but using the same reasoning as above for induction hypothesis~\eqref{eq:ihkip}: the induction hypothesis contained in \( \F \) tells us that the announcement is not perceived by the agent at each modal operator.

Finally, public announcements follow the exact same proof as they did in~\KI{} in
equation~\eqref{eq:kipa},
with the extra information that \( \F([\psi'] \chi) = \F(\psi') \wedge \F(\chi)
\), allowing us to obtain the assumption of the inductive hypothesis in both
inductive hypothesis applications (one for \( \psi' \) and one for \( \chi \)).

For~\TA{}, we assume without loss of generality that \( (M, s) \models
K^\infty_a(\varphi \to \P{a}{\depth{\varphi}}) \wedge \varphi \), this means in
particular the equivalence class of \( (1, s) \) in \( M \mid \varphi \) is \( \{
(1, s'), \; s' \sim_a s, \; (M, s') \models \varphi \} \) since no state
equivalent to \( s \) by \( \sim_a \) has \( a \) not deep enough for \( \varphi
\).
Using once more the same re-writings as in equation~\eqref{eq:rewrta}, it is
sufficient to prove that the depth conditions are the same.
This is the case because \( (M, s) \models \varphi \), therefore by the truth
axiom, \( (M, s) \models \Paphi \).
\end{proof}
\end{document}